\documentclass[runningheads]{llncs}

\usepackage{accv}

\usepackage{accvabbrv}
\usepackage{graphicx}
\usepackage{booktabs}
\usepackage{amsmath,amssymb,mathtools,bm}
\usepackage{xcolor}
\usepackage{placeins}
\usepackage[breaklinks,colorlinks,citecolor=accvblue]{hyperref}
\usepackage{orcidlink}
\usepackage{multirow}

\makeatletter
\def\fps@figure{!tbp}
\def\fps@table{!tbp}
\makeatother
\newcommand{\R}{\mathbb{R}}
\newcommand{\sig}{\sigma}
\newcommand{\argminop}{\mathop{\rm arg\,min}}
\newcommand{\diag}{\operatorname{diag}}

\newcommand{\safeincludegraphics}[2][]{%
  \IfFileExists{#2}{%
    \includegraphics[#1]{#2}%
  }{%
    \fbox{\begin{minipage}[c][0.18\textheight][c]{0.86\linewidth}
      \centering Missing figure: \texttt{#2}
    \end{minipage}}%
  }%
}

\newcommand{\safetableinput}[1]{%
  \IfFileExists{#1}{%
    \input{#1}%
  }{%
    \begin{table}[t]
      \centering
      \small
      \fbox{\begin{minipage}{0.86\linewidth}
      Missing table file: \texttt{#1}.
      \end{minipage}}
      \caption{Placeholder table.}
    \end{table}
  }%
}

\begin{document}

\title{From Dataset Spectral Geometry to Network Weights: A Geometry-Aware Initialization for Sigmoidal MLPs in Image Classification}

\titlerunning{Geometry-Aware Initialization for Sigmoidal MLPs in Image Classification}

\author{Yi-Shan Chu}
\authorrunning{Yi-Shan Chu}
\institute{Institute of Statistical Science, Academia Sinica, Taipei, Taiwan (R.O.C.)\\
\email{easonchu@stat.sinica.edu.tw}}

\maketitle
\begin{abstract}
Classical universal approximation theorems (UAT) establish the expressive power of sigmoidal multilayer perceptrons, but they do not specify how the weights should be initialized. We study a supervised, data-dependent, geometry-aware initialization for one-hidden-layer sigmoidal MLPs that compiles labeled class geometry into network weights. The construction starts from the idea that sigmoid units can act as smooth half-space gates. For each class, we center the training samples at their mean, apply SVD to estimate principal directions and spectral scales, select retained directions by an energy threshold, and represent each retained direction by a pair of sigmoid slab gates. These class-specific gates are then concatenated into a shared hidden layer initialized directly from the training set. We also formulate a SVD-Mahalanobis subspace classifier as a non-neural geometric reference, which tests whether the estimated spectral class geometry is already discriminative before being embedded into the MLP. Experiments on MNIST, Fashion-MNIST, and CIFAR-10 show that the proposed initializer produces a substantially more informative zero-epoch state than a matched task-agnostic Xavier reference, while full training reaches comparable final accuracy. Frozen-hidden experiments and neutral-head ablations further show that the class-wise SVD gates remain useful fixed features, even after the initially aligned output layer is replaced by a Xavier-initialized one.
\end{abstract}

\section{Introduction}

The universal approximation theorem (UAT)~\cite{Cybenko1989} is often understood as an expressive-power statement: with enough hidden units, a sigmoidal MLP can approximate any continuous target on a compact domain. In its classical form, the approximant is a finite sum
\begin{equation}
f_N(x)=\sum_{j=1}^{N}\alpha_j
\sig(w_j^\top x+b_j),\qquad x\in K\subset\R^d ,
\label{eq:uat-form}
\end{equation}
where the parameters are allowed to vary freely. While this form is powerful, it does not specify how the weights should be chosen or initialized in practice. This motivates a natural question: can the geometry of the data provide a more informative starting point for the network?

For image classification, even a simple benchmark such as MNIST exhibits clear class-dependent structure: each digit class has a characteristic mean shape, dominant deformation directions, and less stable residual variations. This suggests that initialization need not be entirely blind to the data. When labeled training samples are available, such class structure can potentially be estimated first and then incorporated into the initial network weights.

This perspective differs from standard task-agnostic initializations such as Xavier/Glorot initialization~\cite{Glorot2010}, whose primary goal is to stabilize signal propagation and gradient scale before training. Our method instead uses the labeled training set to construct a supervised, class-wise, data-dependent initialization tied to the target classification task.

This paper connects two levels of geometry. The first is a compact constructive background for sigmoidal networks. A recent planar sigmoidal construction~\cite{ChuKuo2025UATTropical} shows how half-spaces, polytopes, and finite covers can be organized through a boundary-first, tropical-geometry-inspired viewpoint and compiled into sigmoid gates. We use this component only as geometric motivation: it explains how prescribed regions can be translated into network weights when the geometry is known explicitly.

The main contribution of this paper is the second level: a data-driven high-dimensional instantiation for image data. Instead of assuming specified planar covers, we estimate class geometry from samples. For each class, we compute a centered SVD/PCA basis~\cite{Jolliffe2002}, select an energy-adaptive rank, and formulate a SVD-based subspace classifier. This classifier is not the neural model; it serves as a geometric reference that tests whether class-wise spectral coordinates already separate the classes.

We then compile the same spectral geometry into a one-hidden-layer sigmoidal MLP. Each retained SVD direction defines a normalized coordinate around a class mean. Compatibility with the class is expressed as a slab constraint on this coordinate, and each slab is represented by two sigmoid gates, one for each side of the interval. Thus the hidden layer is partitioned into class-specific groups, with two hidden units per retained direction, and the output layer initially aggregates the gates belonging to each class. This gives an explicit SVD-based initialization whose width is controlled by the retained spectral ranks.

Our contributions are:
\begin{itemize}
\item We provide a concise bridge from UAT finite sums to smooth half-space and cover-based sigmoid gates, building on the planar sigmoidal construction of~\cite{ChuKuo2025UATTropical};
\item We introduce a SVD-based subspace classifier that serves as a non-neural reference for testing whether class-wise spectral geometry is discriminative and for guiding the choice of initialization parameters;
\item We propose a data-driven SVD-based weight initializer for MLPs in which each retained class direction contributes two hidden units and the hidden width is determined by an energy threshold $\tau$;
\item We evaluate the proposed initialization through matched comparisons with Xavier initialization on MNIST, Fashion-MNIST, and CIFAR-10. Full training reaches comparable final accuracy, while zero-epoch, frozen-hidden, and neutral-head ablations isolate the representation value of the class-wise SVD gates.
\end{itemize}

\begin{figure}[!t]
\centering
\safeincludegraphics[width=0.94\linewidth]{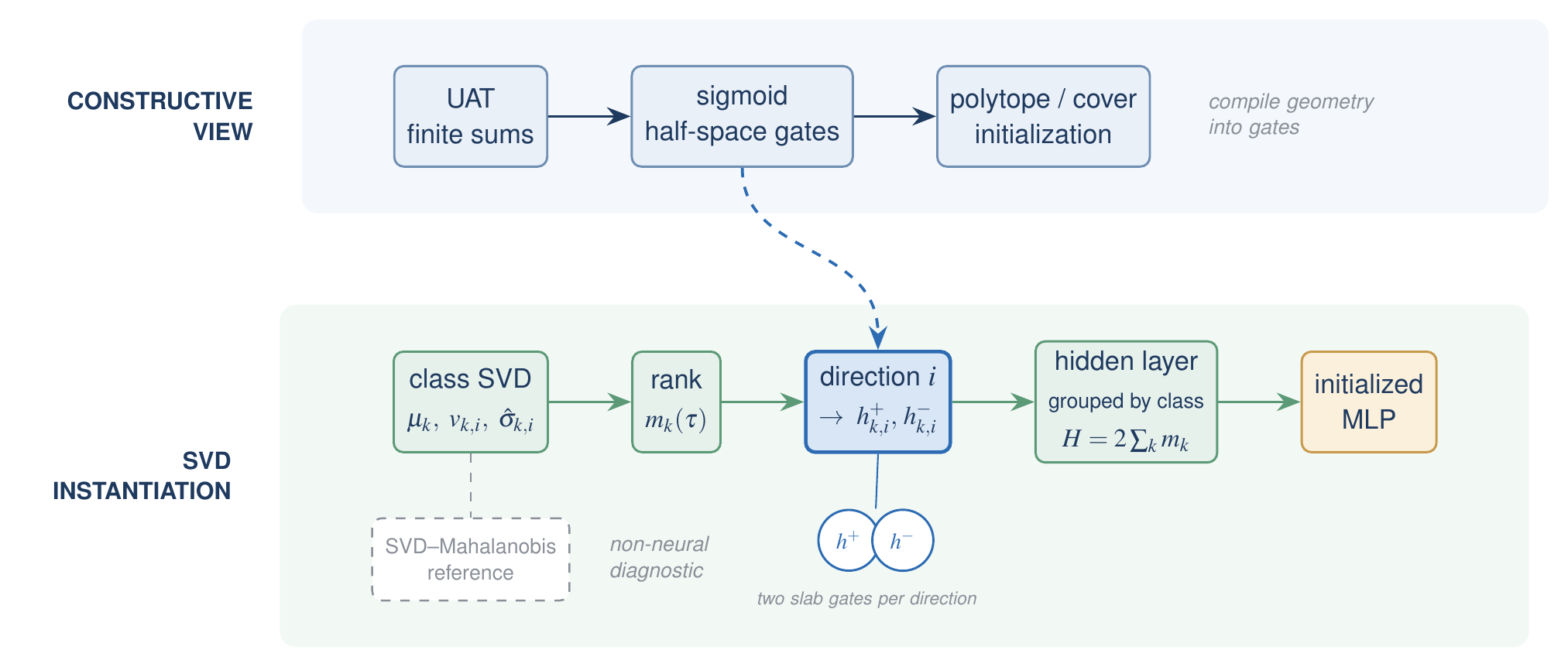}
\caption{Method flow. The top row summarizes the constructive view: finite-sum sigmoidal networks can be organized as half-space gates and then as polytope- or cover-based initializers. The bottom row shows the data-driven instantiation used for high-dimensional images. Class-wise SVD estimates the mean, directions, and scales; an energy threshold selects the retained rank $m_k(\tau)$; each retained direction is compiled into two slab gates $h^+_{k,i}$ and $h^-_{k,i}$; and the resulting gates are placed in a class-grouped hidden layer of width $H=2\sum_k m_k$. The SVD-based subspace classifier is shown only as a non-neural geometric reference for the selection of energy threshold $\tau$.}
\label{fig:method-flow}
\end{figure}

\section{From Finite Sums to Geometric Gates}
\label{sec:uat}

Let \(K\subset\R^d\) be compact and let \(\sig(t)=(1+e^{-t})^{-1}\). Cybenko's theorem states that finite sums of the form \eqref{eq:uat-form} are dense in \(C(K)\) for sigmoidal activations \cite{Cybenko1989}; related and extended results were given by Funahashi \cite{Funahashi1989}, Hornik \etal \cite{Hornik1989}, Hornik \cite{Hornik1991}, Leshno \etal \cite{Leshno1993}, and Barron \cite{Barron1993}; see also Pinkus \cite{Pinkus1999}. These theorems ensure representability, but they do not specify an initialization tied to a particular decision region.

Following the planar construction in~\cite{ChuKuo2025UATTropical}, we use the following elementary gate as the geometric primitive. For \(a\in\R^d\), \(c\in\R\), and sharpness \(\kappa>0\), define
\begin{equation}
  g_{a,c,\kappa}(x)=\sig\bigl(\kappa(c-a^\top x)\bigr).
  \label{eq:halfspace-gate}
\end{equation}
This unit is close to one inside the half-space \(a^\top x\le c\) and close to zero outside it.

\begin{lemma}[Half-space gate]
\label{lem:halfspace}
Let \(\delta>0\). If \(c-a^\top x\ge\delta\), then \(g_{a,c,\kappa}(x)\ge 1-\exp(-\kappa\delta)\). If \(a^\top x-c\ge\delta\), then \(g_{a,c,\kappa}(x)\le \exp(-\kappa\delta)\).
\end{lemma}

\begin{proof}
For \(t\ge 0\), \(\sig(t)=1/(1+e^{-t})\ge 1-e^{-t}\). For \(t\le 0\), \(\sig(t)\le e^t\). Apply these inequalities to \(t=\kappa(c-a^\top x)\). \qed
\end{proof}

Let a polytope be written as
\begin{equation}
  P=\bigcap_{\ell=1}^{m}\{x:a_\ell^\top x\le c_\ell\}.
\end{equation}
Using one gate per supporting half-space, define
\begin{equation}
  G_P(x)=\sum_{\ell=1}^{m}
  \sig\bigl(\kappa(c_\ell-a_\ell^\top x)\bigr)
  -\left(m-\frac{1}{2}\right).
  \label{eq:polytope-score}
\end{equation}

In the following proposition, we say that \(a_\ell^\top x\le c_\ell\) is satisfied with margin \(\delta\) if \(c_\ell-a_\ell^\top x\ge\delta\), and violated with margin \(\delta\) if \(a_\ell^\top x-c_\ell\ge\delta\).

\begin{proposition}
\label{prop:polytope}
Assume \(m\exp(-\kappa\delta)<1/2\). If \(x\) satisfies all inequalities with margin at least \(\delta\), then \(G_P(x)>0\). If \(x\) violates at least one inequality with margin at least \(\delta\), then \(G_P(x)<0\).
\end{proposition}

\begin{proof}
Inside \(P\) with margin \(\delta\), Lemma~\ref{lem:halfspace} gives
\[
G_P(x)\ge m(1-e^{-\kappa\delta})-\left(m-\frac12\right)>0.
\]
If one constraint is violated by margin \(\delta\), then the corresponding gate is at most \(e^{-\kappa\delta}\), while all other gates are at most one. Thus
\[
G_P(x)\le (m-1)+e^{-\kappa\delta}-\left(m-\frac12\right)<0.
\] \qed
\end{proof}

Finite unions are obtained by a second sigmoidal aggregation. For \(C=\bigcup_{r=1}^{R}P_r\), define
\begin{equation}
  F_C(x)=\sum_{r=1}^{R}\sig\bigl(\lambda G_{P_r}(x)\bigr)-\frac12 .
  \label{eq:union-score}
\end{equation}
If \(x\) lies inside one component, one term is close to one; if \(x\) lies outside all components, all terms are close to zero for sufficiently large \(\lambda\). This gives a sigmoidal network in the same finite-sum spirit as UAT, but with weights chosen from an explicit geometric description.

The construction is inspired by the way tropical geometry represents polyhedral structure through max-plus affine pieces \cite{MaclaganSturmfels2015}. ReLU networks inherit closely related piecewise-linear subdivisions and have been studied through tropical rational maps \cite{ZhangNaitzatLim2018}. We do not use ReLU networks; the point is instead to keep a smooth sigmoid activation while borrowing a boundary-first design principle: identify geometric boundary primitives first, then compile them into weights.

For a general compact target region in the plane, one may approximate the region by a finite cover and polygonal components, then apply \eqref{eq:polytope-score} and \eqref{eq:union-score}. The planar examples and visual demonstrations are reported in our prior work~\cite{ChuKuo2025UATTropical}. In this paper, the role of this section is motivational rather than a formal guarantee for the high-dimensional initializer: it explains why prescribed geometry can be compiled into sigmoidal weights, while the image experiments below replace hand-specified geometry by class-wise spectral geometry estimated from data.

\section{Spectral Class Geometry and SVD-Referenced Subspace Classifier}
\label{sec:svd-classifier}

For image data, class regions are not given as planar covers. We therefore estimate class geometry from the training set. For class \(k\), let \(\mu_k\in\R^d\) be the class mean and let \(X_k\in\R^{n_k\times d}\) contain centered samples \(x_i-\mu_k\). Compute
\begin{equation}
  X_k=U_kS_kV_k^\top,\qquad
  V_k=[v_{k,1},\ldots,v_{k,d}],
\end{equation}
where \(S_k=\diag(s_{k,1},\ldots)\). The empirical standard deviation along direction \(v_{k,i}\) is
\begin{equation}
  \widehat{\sigma}_{k,i}=s_{k,i}/\sqrt{n_k-1}.
\end{equation}
For an energy threshold \(\tau\in(0,1)\), choose
\begin{equation}
  m_k(\tau)=\min\left\{m:
  \frac{\sum_{i=1}^{m}s_{k,i}^2}{\sum_i s_{k,i}^2}\ge\tau\right\}.
  \label{eq:energy-rank}
\end{equation}

With \(m_k=m_k(\tau)\), the SVD-Mahalanobis reference score is
\begin{equation}
\begin{split}
  d_k(x) &=
  \sum_{i=1}^{m_k}
  \left(
  \frac{v_{k,i}^\top(x-\mu_k)}{\widehat{\sigma}_{k,i}}
  \right)^2 \\
  &\quad+
  \lambda
  \frac{
  \left\|
  (I-V_{k,m_k}V_{k,m_k}^\top)(x-\mu_k)
  \right\|_2^2
  }{\widehat{\sigma}_{k,m_k}^2},
\end{split}
\label{eq:svd-mahalanobis}
\end{equation}
where \(V_{k,m_k}\) contains the retained directions and \(\lambda\ge 0\) penalizes residual energy outside the retained subspace. 

This score defines a subspace classifier based on the estimated SVD geometry of each class. The first term is a truncated Mahalanobis distance inside the retained class subspace: deviations along high-variance class directions are normalized by their empirical spectral scales. The second term measures the normalized residual energy outside the retained subspace. This complement penalty is important because a sample may have moderate coordinates along the retained directions of a class while still being poorly explained by that class subspace as a whole. The residual term therefore acts as a rejection mechanism for samples that lie far from the retained spectral model. Empirically, removing this term by setting $\lambda=0$ leads to a large accuracy drop of more than 13 percentage points, as shown in Fig.~\ref{fig:svd-reference}.

Classification is by
\begin{equation}
  \widehat y(x)=\argminop_k d_k(x).
\end{equation}

The role of this classifier is twofold in this work. First, it provides a non-neural reference for testing whether class-wise spectral geometry is already discriminative. Second, it identifies the retained directions, ranks, and spectral scales that will later be compiled into sigmoid slab gates. In this sense, the classifier is not a competing neural architecture, but a geometric intermediate between the estimated class subspaces and the initialized MLP.

\begin{figure}[!t]
\centering
\safeincludegraphics[width=0.66\linewidth]{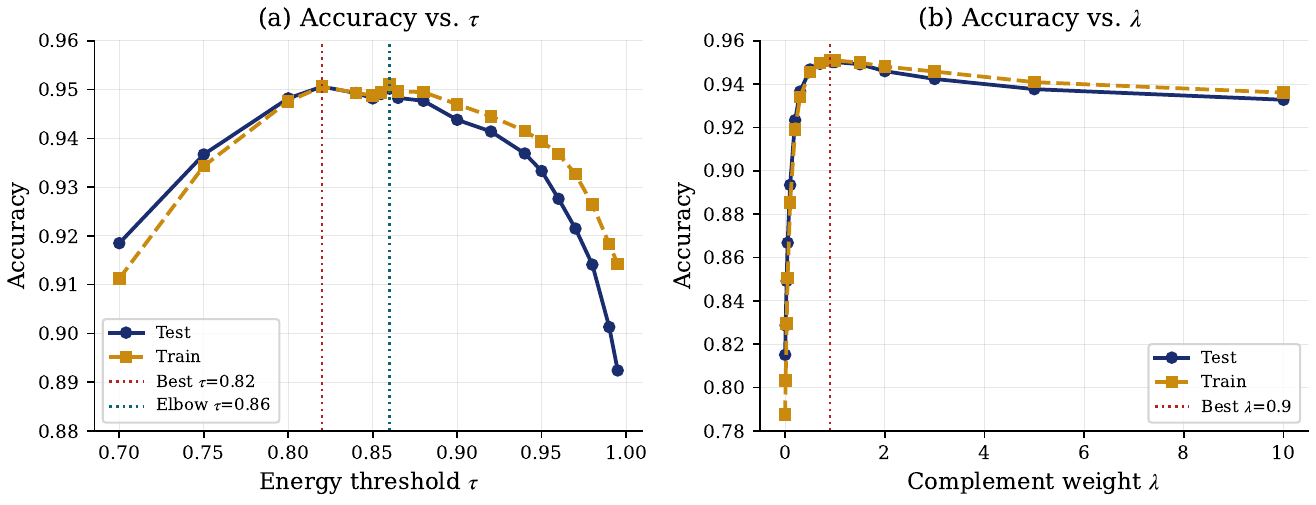}
\safeincludegraphics[width=0.32\linewidth]{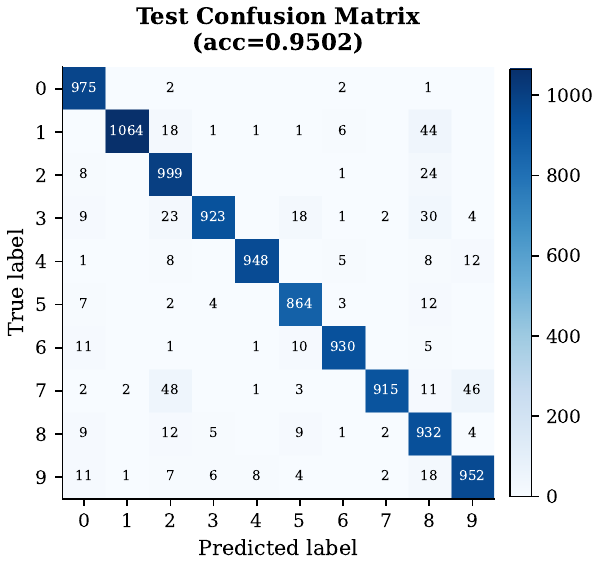}
\caption{SVD-based Subspace Classifier reference for MNIST. The energy threshold selects class-dependent ranks, and the complement penalty rejects samples that are not well explained by a class subspace. This reference is non-neural; it is used to validate the spectral geometry before compiling it into an MLP.}
\label{fig:svd-reference}
\end{figure}

\section{Compiling Spectral Slabs into a Sigmoid MLP}
\label{sec:svd-init}

We now compile the class-wise SVD geometry into a one-hidden-layer sigmoid MLP. For class \(k\) and retained direction \(i\), define the normalized spectral coordinate
\begin{equation}
  z_{k,i}(x)=
  \frac{v_{k,i}^\top(x-\mu_k)}
       {\widehat{\sigma}_{k,i}}.
\end{equation}
A sample compatible with class \(k\) should have moderate values of \(z_{k,i}(x)\) along the retained directions. We model the interval \([-\rho,\rho]\) by two sigmoid slab gates:
\begin{equation}
  h_{k,i}^{+}(x)=\sig\bigl(\beta(\rho-z_{k,i}(x))\bigr),\qquad
  h_{k,i}^{-}(x)=\sig\bigl(\beta(\rho+z_{k,i}(x))\bigr),
  \label{eq:slab-gates}
\end{equation}
where \(\rho>0\) is the slab half-width and \(\beta>0\) controls sharpness. Both gates are near one when \(z_{k,i}(x)\in[-\rho,\rho]\).

Equation~\eqref{eq:slab-gates} gives explicit hidden-layer weights:
\begin{equation}
\begin{array}{ll}
  w_{k,i}^{+}= -\beta v_{k,i}/\widehat{\sigma}_{k,i},
  & b_{k,i}^{+}= \beta\rho+
    \beta v_{k,i}^\top\mu_k/\widehat{\sigma}_{k,i},\\[2mm]
  w_{k,i}^{-}= \phantom{-}\beta v_{k,i}/\widehat{\sigma}_{k,i},
  & b_{k,i}^{-}= \beta\rho-
    \beta v_{k,i}^\top\mu_k/\widehat{\sigma}_{k,i}.
\end{array}
\label{eq:explicit-svd-weights}
\end{equation}
Thus each retained SVD direction contributes two hidden units: one gate for the upper half-space constraint \(z_{k,i}\le \rho\) and one gate for the lower half-space constraint \(z_{k,i}\ge -\rho\). The hidden layer is partitioned into class-specific groups,
\[
\mathcal H= \mathcal H_0\,|\,\mathcal H_1\,|\,\cdots\,|\,\mathcal H_9, \qquad |\mathcal H_k|=2m_k(\tau),
\]
and its total width is
\begin{equation}
  H(\tau)=2\sum_{k=0}^{9}m_k(\tau).
  \label{eq:hidden-width}
\end{equation}
The initialized class logit is the normalized average of the class slab responses:
\begin{equation}
  \ell_k^{(0)}(x)
  =
  \frac{\gamma}{2m_k}
  \sum_{i=1}^{m_k}
  \bigl(h_{k,i}^{+}(x)+h_{k,i}^{-}(x)\bigr),
  \label{eq:initial-logit}
\end{equation}
with scale \(\gamma>0\). The complete analytical initialization therefore includes both the SVD-constructed hidden gates and the class-aligned output aggregation in \eqref{eq:initial-logit}. This is distinct from the non-neural SVD-Mahalanobis score in \eqref{eq:svd-mahalanobis}: Eq.~\eqref{eq:svd-mahalanobis} uses squared normalized distances and a complement residual as a reference classifier, whereas \eqref{eq:slab-gates}--\eqref{eq:initial-logit} compile retained-direction slab constraints into bounded sigmoid gates and an initialized MLP.

\subsection{Experiment Protocols}
After initialization, we study two protocols: \emph{full training}, in which all weights are updated, and \emph{frozen-hidden training}, in which the hidden layer is fixed and only the output layer is trained. For each dataset and each threshold \(\tau\), the Xavier reference uses exactly the same architecture and hidden width \(H(\tau)\). Xavier is task-agnostic, while the proposed initialization is supervised and class-wise; we therefore use it as a matched random-initialization reference rather than as a supervised competitor.

The zero-epoch accuracy of the proposed model evaluates the complete analytical initialization, including both the SVD hidden gates and the class-aligned output head. To separate this output-head effect from the hidden representation, we also run neutral-head frozen ablations: the same SVD gates are kept fixed, but the output head is reinitialized with Xavier and then trained. Additional frozen controls replace class-wise SVD directions by class-centered random directions or global PCA directions, remove spectral scaling, use class-mean-only gates, or use frozen Xavier random gates. These ablations isolate whether the observed frozen-hidden performance comes from the class-wise spectral gates rather than from the aligned output head, class centering alone, generic PCA geometry, or random features.

For a one-hidden-layer sigmoid MLP with input dimension \(d\), hidden width \(H\), and \(10\) outputs, the total parameter count is
\begin{equation}
  (dH+H)+(10H+10)=(d+11)H+10.
\end{equation}
For MNIST and Fashion-MNIST, \(d=784\); for CIFAR-10, \(d=3072\). Under frozen-hidden training, the trainable parameter count is only
\begin{equation}
  10H+10.
\end{equation}

\section{Experiments}
\label{sec:exp}

\subsection{Datasets and Experimental Setup}
\label{sec:datasets-protocol}

We evaluate the proposed SVD-based initialization on three image-classification datasets. The main controlled study uses MNIST handwritten digits~\cite{LeCun1998,LeCunMNIST}, flattened to \(x\in\R^{784}\). To test whether the effect is specific to digit shapes, we also evaluate the same protocol on Fashion-MNIST~\cite{Xiao2017FashionMNIST}, which has the same image size and number of classes as MNIST but contains clothing categories. Finally, we include CIFAR-10~\cite{Krizhevsky2009CIFAR} as a more challenging natural-image stress test. For CIFAR-10, we deliberately apply the method directly to flattened raw RGB pixels, \(x\in\R^{3072}\), without convolutional inductive bias, data augmentation, or pretrained features.

All neural experiments use one-hidden-layer sigmoidal MLPs trained for 20 epochs with Adam~\cite{KingmaBa2015}. We test the energy thresholds \(\tau\in\{0.75,0.86,0.90,0.95\}\). For each dataset and threshold, the proposed SVD-initialized model and the Xavier reference use the same hidden width \(H(\tau)\), optimizer, number of epochs, and training protocol. Specifically, the SVD-initialized MLP uses slab half-width \(\rho=2.0\), sharpness \(\beta=4.0\), and initial logit scale \(\gamma=8.0\). The SVD-based subspace classifier uses complement weight \(\lambda=1.0\). All neural models use Xavier uniform initialization for random weights, zero biases, batch size 256, learning rate \(10^{-3}\), and no weight decay; evaluation uses batch size 1024. Results are averaged over five random seeds: \(0,1,2,3,4\). All SVD quantities, class means, directions, scales, and ranks are computed only from the training splits; test sets are used only for reporting, so there is no test-set leakage. For the main focused comparison, we use the common threshold \(\tau=0.95\) across all three datasets. The full threshold sweep \(\tau\in\{0.75,0.86,0.90,0.95\}\) is reported separately to show the sensitivity to retained spectral complexity.

Tables~\ref{tab:all-energy-rank}--\ref{tab:head-neutral-ablation} summarize the experimental evidence. Table~\ref{tab:all-energy-rank} reports the retained ranks and induced hidden widths. Table~\ref{tab:all-geometry-reference} compares the non-neural SVD-Mahalanobis reference with the zero-epoch initialized MLP. Tables~\ref{tab:all-focus-matched-protocols} and~\ref{tab:all-threshold-protocols} report the matched train/freeze comparisons at the highlight threshold and across thresholds, respectively. Table~\ref{tab:head-neutral-ablation} reports the head-neutral frozen ablations at the common threshold \(\tau=0.95\).

\begin{table}[t]
\centering
\scriptsize
\setlength{\tabcolsep}{3.0pt}
\begin{tabular}{llrrrrr}
\toprule
Dataset & $\tau$ & Mean rank & Rank range & $H$ & Mean retained energy & Mean cutoff std. \\
\midrule
\multirow{4}{*}{MNIST}
& 0.75 & 23.9 & 9--32 & 478 & 0.753 & 0.5643 \\
& 0.86 & 48.4 & 23--60 & 968 & 0.861 & 0.3348 \\
& 0.90 & 68.2 & 36--82 & 1364 & 0.901 & 0.2534 \\
& 0.95 & 121.4 & 71--142 & 2428 & 0.950 & 0.1558 \\
\cmidrule(lr){1-7}
\multirow{4}{*}{Fashion-MNIST}
& 0.75 & 18.6 & 12--41 & 372 & 0.754 & 0.5446 \\
& 0.86 & 50.2 & 24--96 & 1004 & 0.861 & 0.2890 \\
& 0.90 & 78.2 & 36--135 & 1564 & 0.901 & 0.2121 \\
& 0.95 & 155.5 & 73--222 & 3110 & 0.950 & 0.1281 \\
\cmidrule(lr){1-7}
\multirow{4}{*}{CIFAR-10}
& 0.75 & 22.3 & 15--34 & 446 & 0.752 & 1.0229 \\
& 0.86 & 61.6 & 44--86 & 1232 & 0.860 & 0.5210 \\
& 0.90 & 98.6 & 74--131 & 1972 & 0.900 & 0.3812 \\
& 0.95 & 206.8 & 164--267 & 4136 & 0.950 & 0.2196 \\
\bottomrule
\end{tabular}
\caption{Energy thresholds determine the class-dependent SVD ranks used by both the SVD-based subspace classifier and the MLP initializer. The hidden width is $H=2\sum_k m_k$ because each retained direction gives two sigmoid slab gates. MNIST and Fashion-MNIST are flattened to $\mathbb{R}^{784}$, while CIFAR-10 is flattened to $\mathbb{R}^{3072}$.}
\label{tab:all-energy-rank}
\end{table}

\begin{table}[t]
\centering
\scriptsize
\setlength{\tabcolsep}{4.0pt}
\begin{tabular}{llrrr}
\toprule
Dataset & $\tau$ & $H$ & SVD-Mahalanobis(\%) & Ours init(\%) \\
\midrule
\multirow{4}{*}{MNIST}
& 0.75 & 478 & 93.67 & 45.53 \\
& 0.86 & 968 & 95.02 & 81.35 \\
& 0.90 & 1364 & 94.38 & 86.17 \\
& 0.95 & 2428 & 93.33 & 87.41 \\
\cmidrule(lr){1-5}
\multirow{4}{*}{Fashion-MNIST}
& 0.75 & 372 & 67.42 & 56.34  \\
& 0.86 & 1004 & 74.43 & 62.54  \\
& 0.90 & 1564 & 78.41 & 66.91  \\
& 0.95 & 3110 & 83.03 & 71.32  \\
\cmidrule(lr){1-5}
\multirow{4}{*}{CIFAR-10}
& 0.75 & 446 & 23.20 & 20.49  \\
& 0.86 & 1232 & 24.14 & 20.39  \\
& 0.90 & 1972 & 27.19 & 20.32  \\
& 0.95 & 4136 & 32.46 & 18.51  \\
\bottomrule
\end{tabular}
\caption{Geometry-only performance before gradient updates. The SVD-Mahalanobis column reports the non-neural SVD-based subspace reference classifier using Eq.~\eqref{eq:svd-mahalanobis}. Ours init reports the zero-epoch test accuracy of the proposed SVD-based initialized MLP.}
\label{tab:all-geometry-reference}
\end{table}

\begin{table}[t]
\centering
\scriptsize
\setlength{\tabcolsep}{2.5pt}
\begin{tabular}{llrrrr}
\toprule
Dataset & Model  & $H$ & Trainable params & Init acc.(\%) & Final acc.(\%) \\
\midrule
\multirow{4}{*}{MNIST}
& Ours (full) & 2428 & 1,930,270 & 87.41 & 98.09 $\pm$ 0.07 \\
& Xavier (full) & 2428 & 1,930,270 & 10.42 $\pm$ 0.78 & 97.99 $\pm$ 0.04 \\
& Ours (frozen) & 2428 & 24,290 & 87.41 & 97.15 $\pm$ 0.04 \\
& Xavier (frozen) & 2428 & 24,290 & 10.42 $\pm$ 0.78 & 89.62 $\pm$ 0.28 \\
\cmidrule(lr){1-6}
\multirow{4}{*}{Fashion-MNIST}
& Ours (full) & 3110 & 2,472,460 & 71.32 & 88.99 $\pm$ 0.73 \\
& Xavier (full) & 3110 & 2,472,460 & 10.01 $\pm$ 0.01 & 87.96 $\pm$ 0.54 \\
& Ours (frozen) & 3110 & 31,110 & 71.32 & 86.50 $\pm$ 0.64 \\
& Xavier (frozen) & 3110 & 31,110 & 10.01 $\pm$ 0.01 & 79.51 $\pm$ 0.82 \\
\cmidrule(lr){1-6}
\multirow{4}{*}{CIFAR-10}
& Ours (full) & 4136 & 12,751,298 & 18.51 & 52.85 $\pm$ 0.68 \\
& Xavier (full) & 4136 & 12,751,298 & 10.18 $\pm$ 1.45 & 51.44 $\pm$ 0.59 \\
& Ours (frozen) & 4136 & 41,370 & 18.51 & 44.75 $\pm$ 1.06 \\
& Xavier (frozen) & 4136 & 41,370 & 10.18 $\pm$ 1.45 & 35.44 $\pm$ 0.66 \\
\bottomrule
\end{tabular}
\caption{Matched MLP comparison at the common highlight threshold $\tau=0.95$. Within each dataset, SVD and Xavier models use the same architecture and hidden width; they differ only in initialization and whether the hidden layer is trainable. The frozen rows test the quality of the initialized representation itself.}
\label{tab:all-focus-matched-protocols}
\end{table}

\begin{table}[t]
\centering
\scriptsize
\setlength{\tabcolsep}{2.2pt}
\begin{tabular}{llrrrrr}
\toprule
Dataset & $\tau$ & $H$ & Ours full(\%) & Xavier full(\%) & Ours frozen(\%) & Xavier frozen(\%) \\
\midrule
\multirow{4}{*}{MNIST}
& 0.75 & 478 & 97.47 $\pm$ 0.06 & 97.88 $\pm$ 0.05 & 91.33 $\pm$ 0.17 & 88.59 $\pm$ 0.10 \\
& 0.86 & 968 & 97.80 $\pm$ 0.04 & 97.97 $\pm$ 0.05 & 95.01 $\pm$ 0.12 & 89.23 $\pm$ 0.19 \\
& 0.90 & 1364 & 97.88 $\pm$ 0.04 & 98.00 $\pm$ 0.03 & 95.97 $\pm$ 0.19 & 89.37 $\pm$ 0.29 \\
& 0.95 & 2428 & 98.09 $\pm$ 0.07 & 97.99 $\pm$ 0.04 & 97.15 $\pm$ 0.04 & 89.62 $\pm$ 0.28 \\
\cmidrule(lr){1-7}
\multirow{4}{*}{Fashion-MNIST}
& 0.75 & 372 & 88.08 $\pm$ 0.16 & 88.24 $\pm$ 0.28 & 79.08 $\pm$ 0.14 & 78.76 $\pm$ 0.16 \\
& 0.86 & 1004 & 88.39 $\pm$ 0.21 & 88.17 $\pm$ 0.26 & 84.04 $\pm$ 0.27 & 80.00 $\pm$ 0.23 \\
& 0.90 & 1564 & 88.48 $\pm$ 0.52 & 88.09 $\pm$ 0.35 & 85.30 $\pm$ 0.65 & 80.24 $\pm$ 0.32 \\
& 0.95 & 3110 & 88.99 $\pm$ 0.73 & 87.96 $\pm$ 0.54 & 86.50 $\pm$ 0.64 & 79.51 $\pm$ 0.82 \\
\cmidrule(lr){1-7}
\multirow{4}{*}{CIFAR-10}
& 0.75 & 446 & 51.46 $\pm$ 0.27 & 50.84 $\pm$ 0.58 & 39.72 $\pm$ 0.41 & 35.60 $\pm$ 0.47 \\
& 0.86 & 1232 & 52.90 $\pm$ 0.52 & 51.88 $\pm$ 0.53 & 42.83 $\pm$ 0.95 & 36.24 $\pm$ 1.12 \\
& 0.90 & 1972 & 52.97 $\pm$ 0.70 & 52.23 $\pm$ 0.48 & 43.46 $\pm$ 0.64 & 36.36 $\pm$ 1.23 \\
& 0.95 & 4136 & 52.85 $\pm$ 0.68 & 51.44 $\pm$ 0.59 & 44.75 $\pm$ 1.06 & 35.44 $\pm$ 0.66 \\
\bottomrule
\end{tabular}
\caption{Effect of the energy threshold on matched-protocol accuracy across datasets. Fully trainable models reach similar final accuracy, while the frozen-hidden comparison consistently favors the proposed SVD slab gates over frozen random gates. CIFAR-10 is evaluated directly on flattened raw pixels and therefore serves as a natural-image stress test.}
\label{tab:all-threshold-protocols}
\end{table}

\begin{table}[!t]
\centering
\scriptsize
\setlength{\tabcolsep}{2.5pt}
\begin{tabular}{lccc}
\toprule
Frozen-hidden feature map & MNIST & Fashion-MNIST & CIFAR-10 \\
\midrule
SVD gates + aligned head (ours) & 87.41/97.15 & 71.32/86.50 & 18.51/44.75 \\
SVD gates (ours) + Xavier head & 10.03/96.99 & 10.02/86.32 & 9.46/44.19 \\
Class-centered random directions + Xavier head & 10.12/94.55 & 9.60/85.45 & 9.69/41.82 \\
Global PCA directions + Xavier head & 10.31/89.49 & 9.78/79.66 & 9.53/34.65 \\
SVD without scale normalization + Xavier head & 10.11/90.80 & 9.51/73.73 & 9.83/38.57 \\
Class-mean-only gates + Xavier head & 11.20/74.48 & 11.32/62.22 & 10.51/28.29 \\
Xavier random gates & 10.42/89.62 & 10.01/79.51 & 10.18/35.44 \\
\bottomrule
\end{tabular}
\caption{Head-neutral frozen-hidden ablations at the common threshold $\tau=0.95$. Each cell reports mean init/final test accuracy (\%) over five random seeds. Only the output head is trained. The first row uses the complete analytical initialization with the aligned head, while the remaining rows use a neutral Xavier output head.}
\label{tab:head-neutral-ablation}
\end{table}

\subsection{Results}
\label{sec:results}

The key comparison is not SVD-based initialization against a smaller or larger network; it is the proposed supervised, class-wise initialization against matched task-agnostic Xavier references with the same architecture. Across datasets, the complete analytical initialization starts from a substantially more informative state than Xavier initialization. At the common threshold \(\tau=0.95\), it reaches \(87.41\%\), \(71.32\%\), and \(18.51\%\) zero-epoch test accuracy on MNIST, Fashion-MNIST, and CIFAR-10, respectively, while the matched Xavier models remain near chance level.

Under full training, SVD-initialized and Xavier-initialized models reach similar final accuracy, and Xavier can be slightly better at some thresholds. We therefore do not claim improved final accuracy from the initializer. Instead, full training is used as a convergence diagnostic: at \(\tau=0.95\), using fixed reference targets of \(97\%\), \(88\%\), and \(50\%\) for MNIST, Fashion-MNIST, and CIFAR-10, respectively, and averaging first-hit epochs over five seeds, the SVD initialization reaches these targets in \(6.2/8.2/8.4\) epochs, compared with \(10.2/16.0/12.6\) epochs for Xavier.

The frozen-hidden setting is the main representation test: when the hidden layer cannot be trained, strong performance indicates that the initialized gates already form useful features. At the common threshold \(\tau=0.95\), frozen SVD gates reach \(97.15\%\), \(86.50\%\), and \(44.75\%\) on MNIST, Fashion-MNIST, and CIFAR-10, respectively, compared with \(89.62\%\), \(79.51\%\), and \(35.44\%\) for the matched frozen Xavier models. This is consistent with the interpretation in~\cite{ChuKuo2025UATTropical} that geometry-informed initialization can shift part of training from discovering decision boundaries from scratch toward calibrating and refining an already
structured representation.

Table~\ref{tab:head-neutral-ablation} separates the hidden gates from the aligned output head. At the common threshold \(\tau=0.95\), replacing the aligned head by a Xavier head drops initialization accuracy to chance level, as expected, but barely changes final frozen-hidden accuracy: \(97.15\%\) to \(96.99\%\) on MNIST, \(86.50\%\) to \(86.32\%\) on Fashion-MNIST, and \(44.75\%\) to \(44.19\%\) on CIFAR-10. Thus, the zero-epoch advantage belongs to the complete analytical initialization, while the final frozen-hidden advantage is not due to the aligned head alone. The remaining controls also trail the SVD gates with a neutral output head, indicating that class-wise spectral structure matters.

\begin{figure}[!tbp]
\centering
\safeincludegraphics[width=0.3\linewidth]{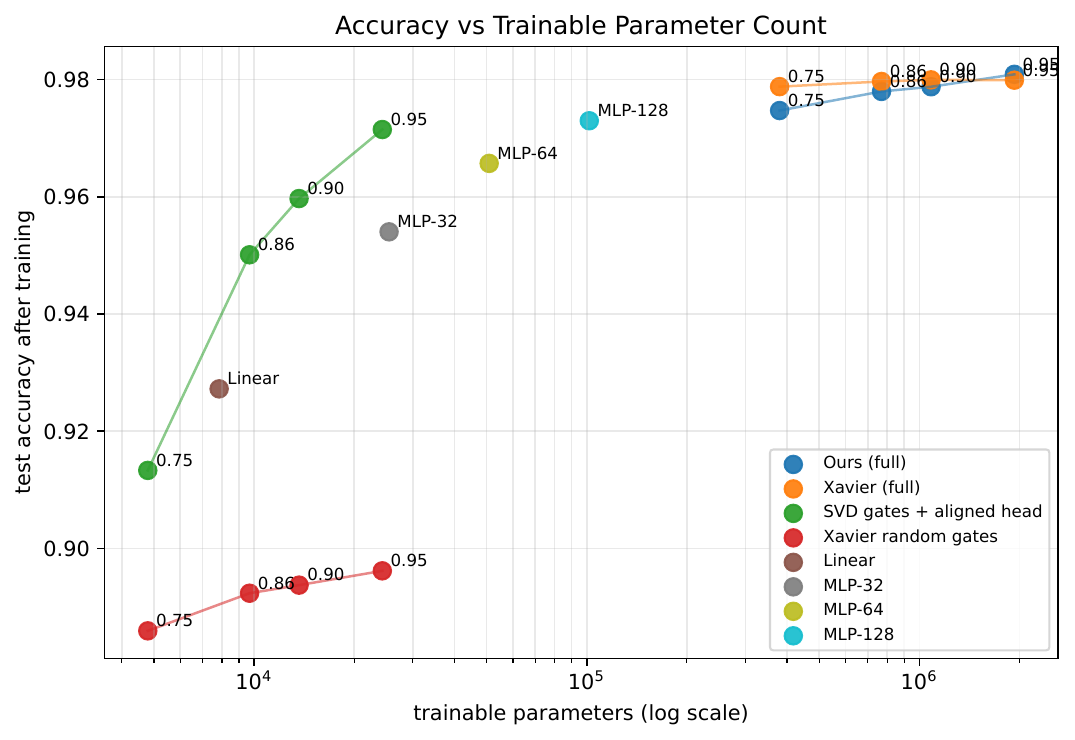}
\safeincludegraphics[width=0.3\linewidth]{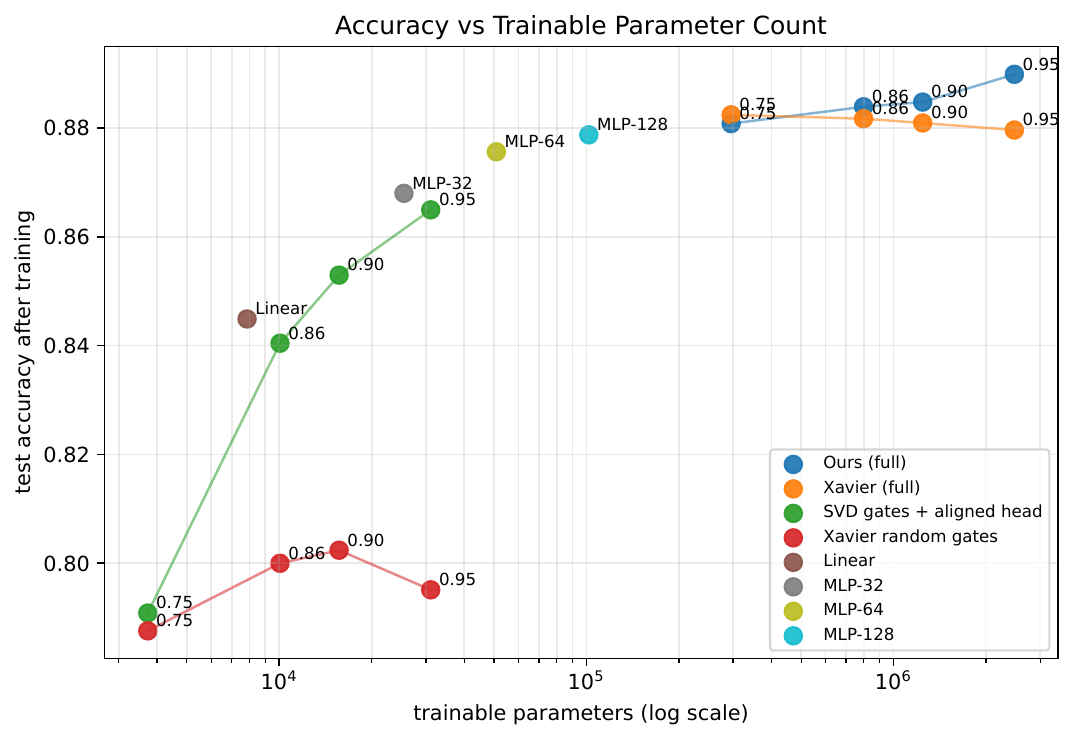}
\safeincludegraphics[width=0.3\linewidth]{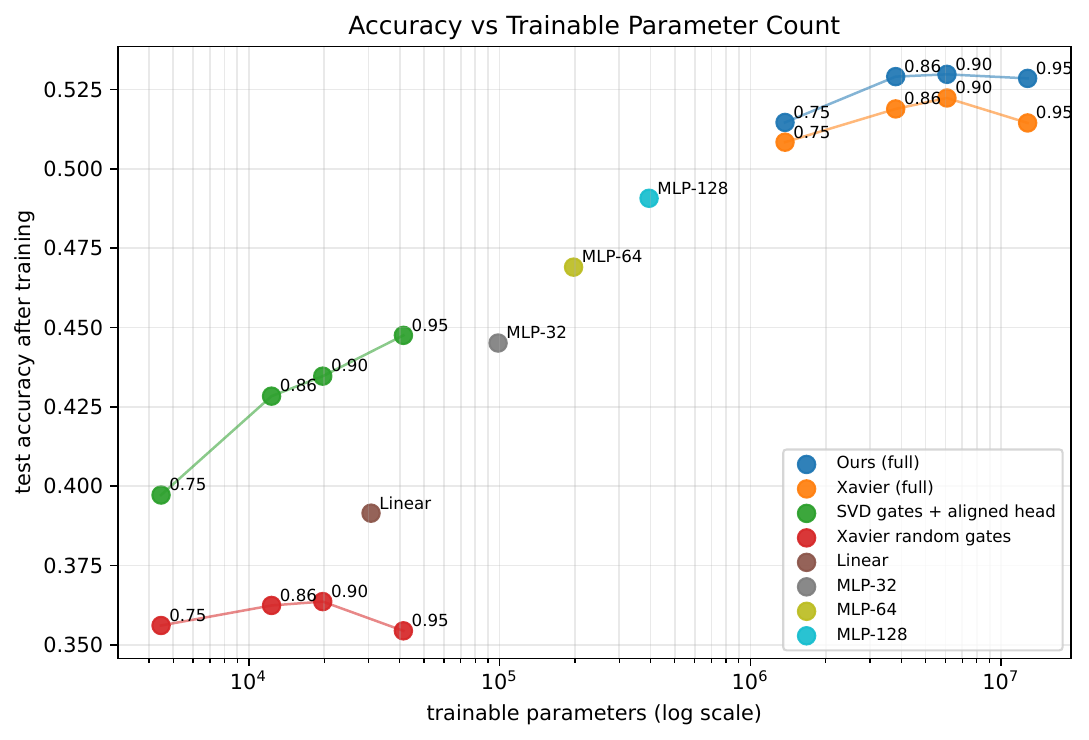}
\caption{Accuracy versus trainable parameter count on MNIST, Fashion-MNIST, and CIFAR-10, shown from left to right. Ours and Xavier use matched hidden width at each energy threshold. Fully trainable models reach similar final accuracy, while the frozen-hidden comparison evaluates the quality of the fixed hidden representation.}
\label{fig:param-count}
\end{figure}

\begin{figure}[!tbp]
\centering
\safeincludegraphics[width=0.3\linewidth]{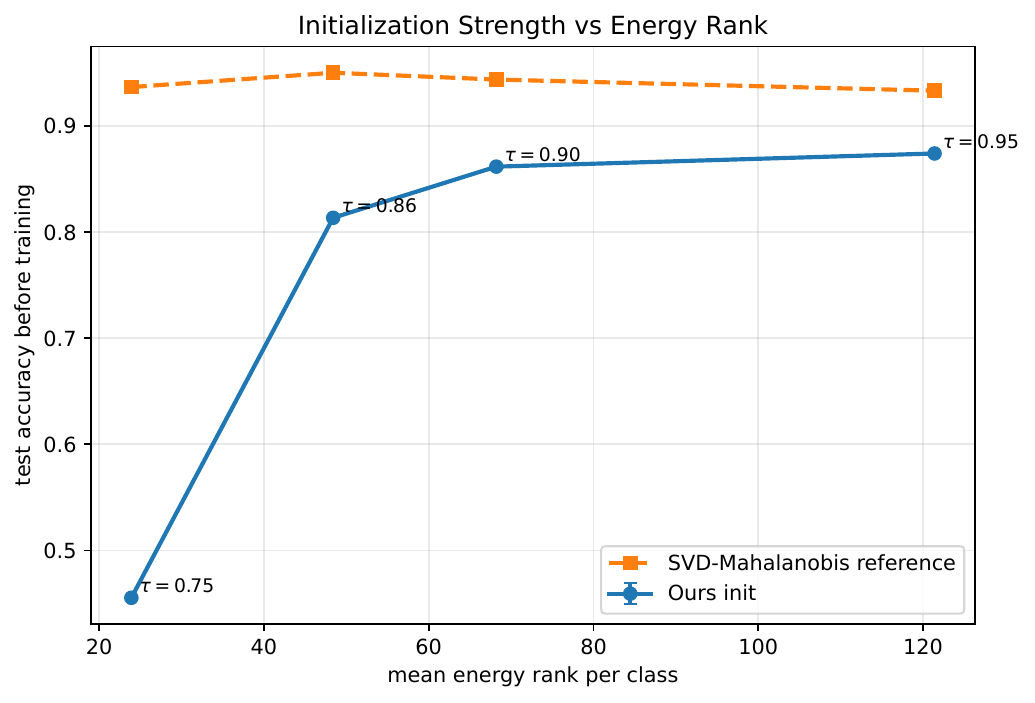}
\safeincludegraphics[width=0.3\linewidth]{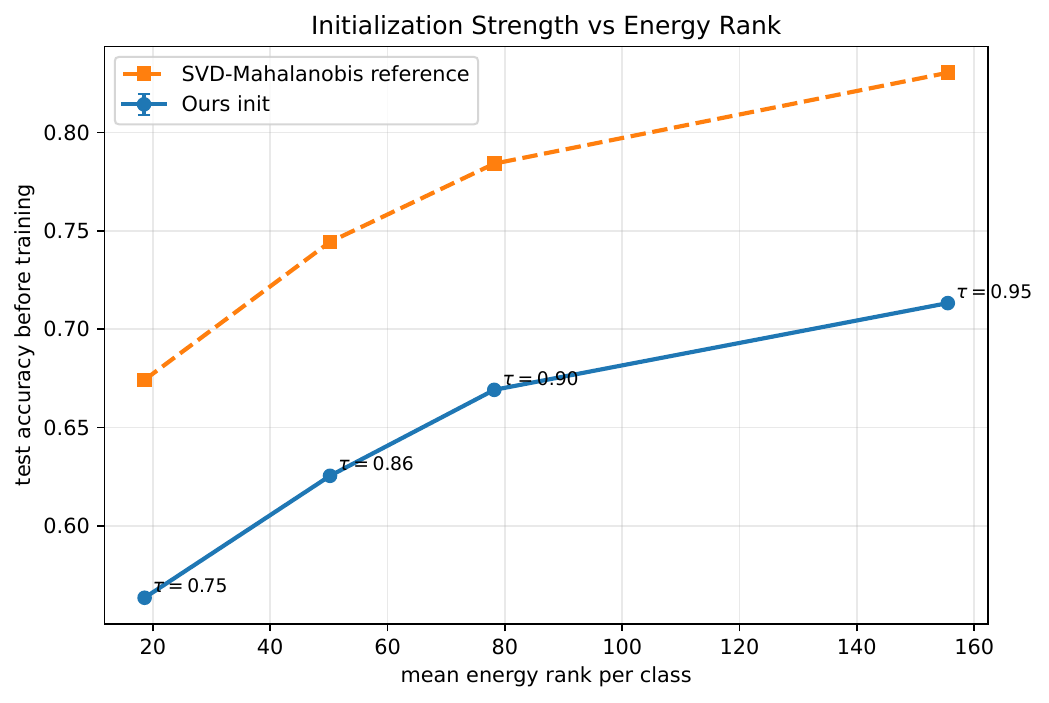}
\safeincludegraphics[width=0.3\linewidth]{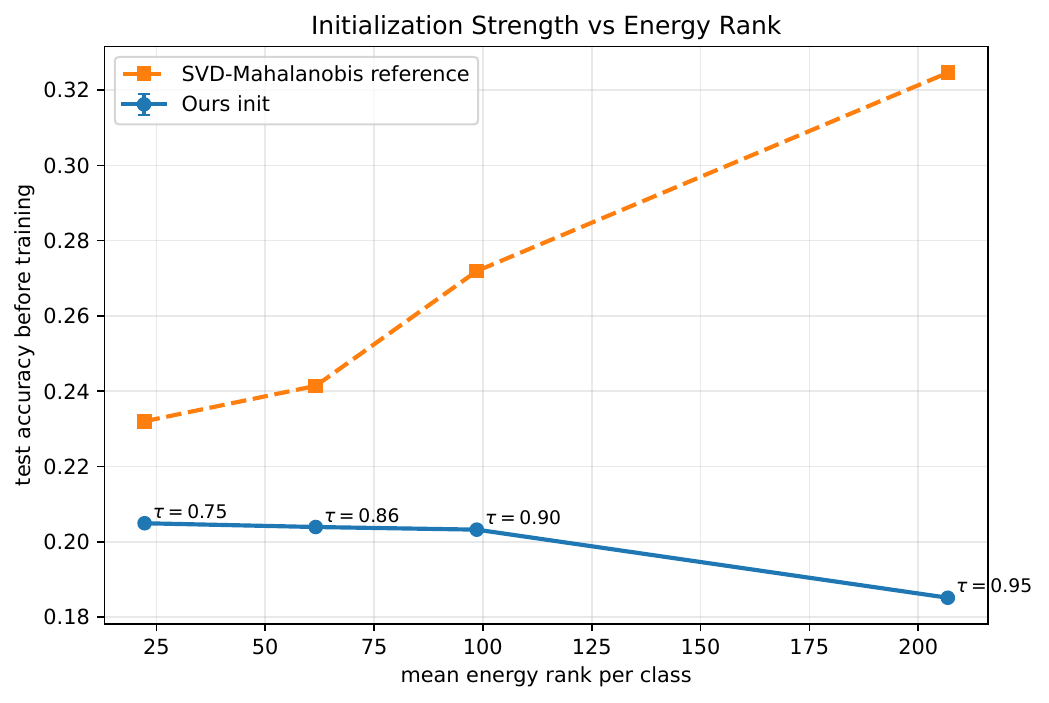}
\safeincludegraphics[width=0.3\linewidth]{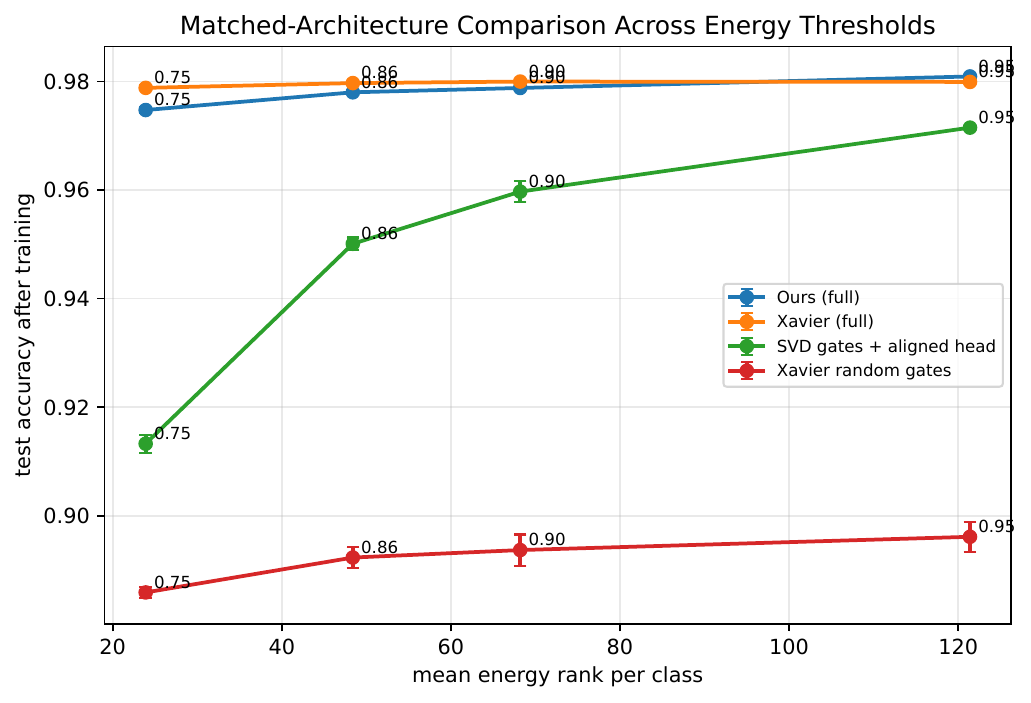}
\safeincludegraphics[width=0.3\linewidth]{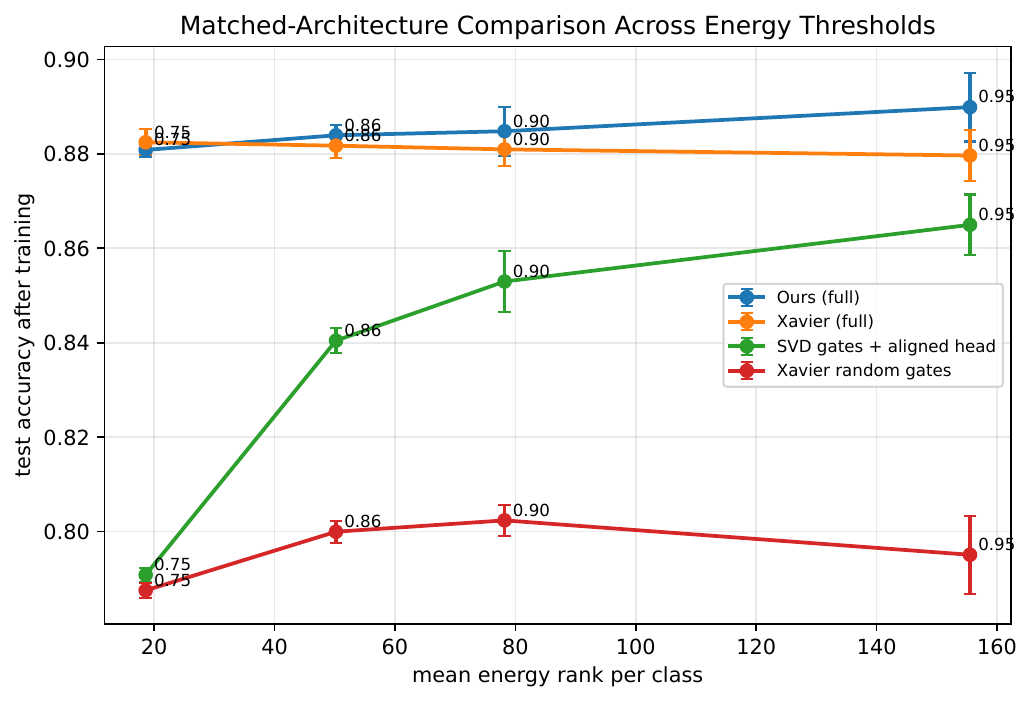}
\safeincludegraphics[width=0.3\linewidth]{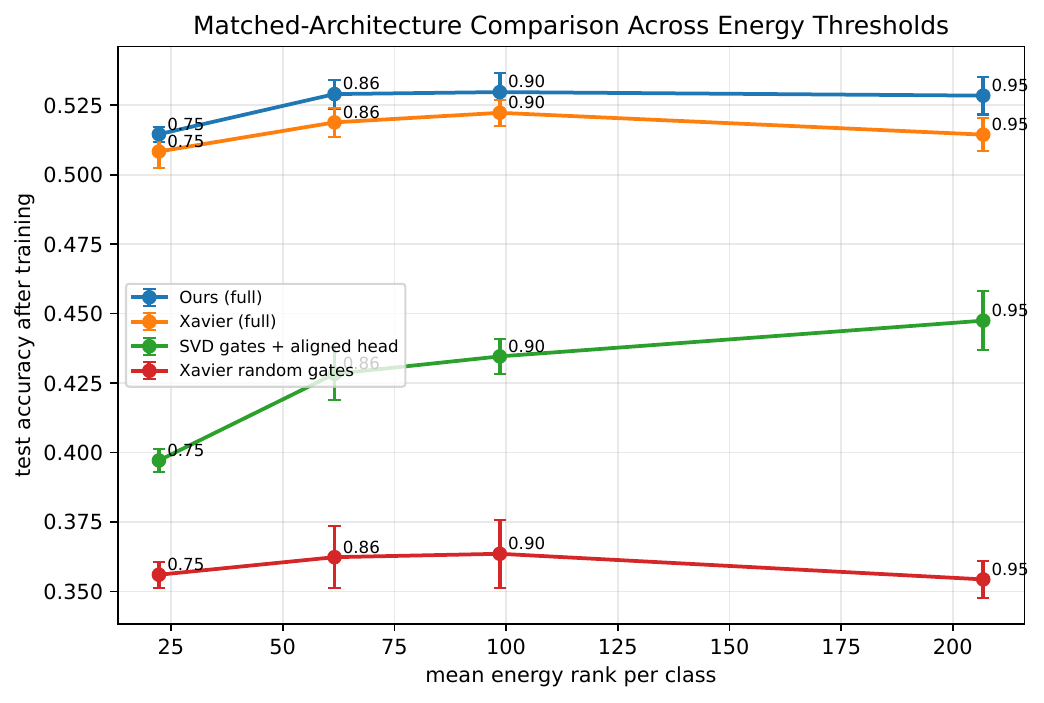}
\caption{Energy-rank analysis on MNIST, Fashion-MNIST, and CIFAR-10, shown from left to right. Top: before training, the zero-epoch SVD-initialized MLP is compared with the non-neural SVD-based subspace classifier across energy ranks. Bottom: after training, fully trainable models are close, while the frozen-hidden setting reveals the representation value of the class-wise SVD gates.}
\label{fig:energy-analysis}
\end{figure}

\begin{figure}[!tbp]
\centering
\safeincludegraphics[width=0.3\linewidth]{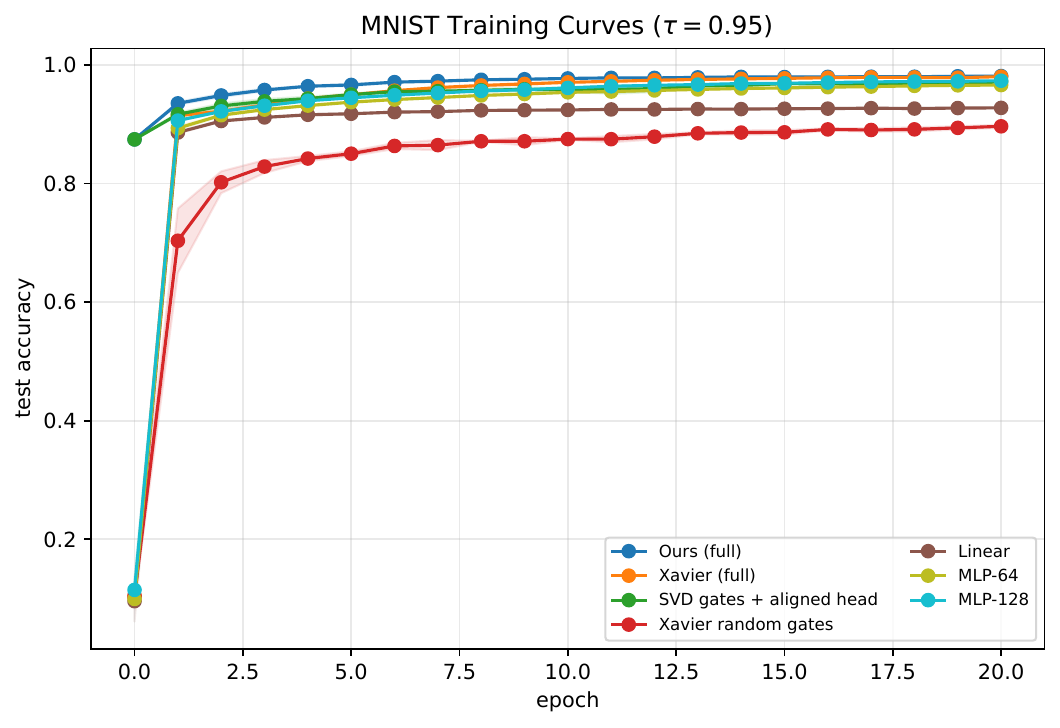}
\safeincludegraphics[width=0.3\linewidth]{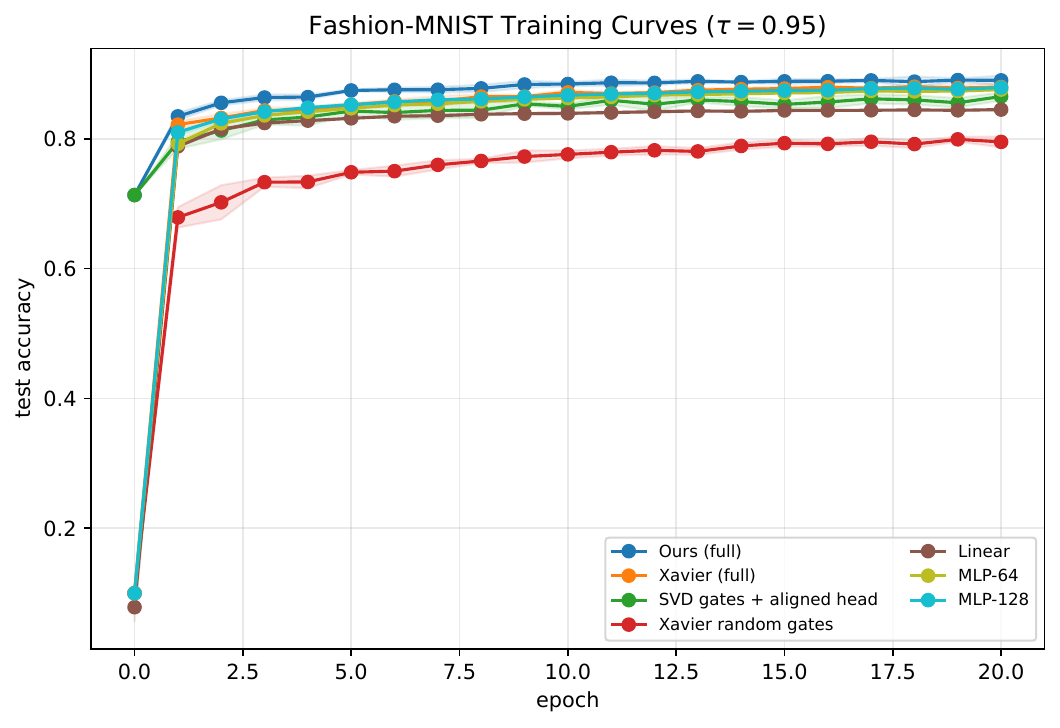}
\safeincludegraphics[width=0.3\linewidth]{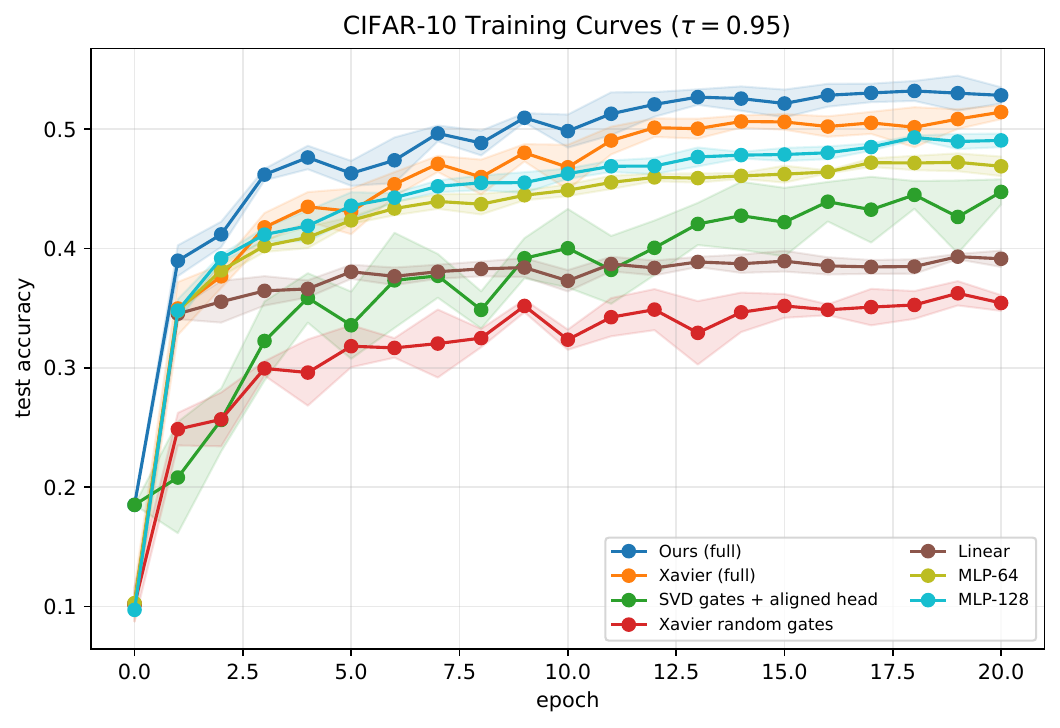}
\caption{Training curves on MNIST, Fashion-MNIST, and CIFAR-10, shown from left to right, all at the common threshold \(\tau=0.95\). Across datasets, the SVD-based initialization starts from a more informative state than the matched task-agnostic Xavier reference. Under full training, the models reach similar final accuracy, while the frozen-hidden protocol highlights the representation quality of the proposed initialization.}
\label{fig:training-curves}
\end{figure}

Figs.~\ref{fig:param-count}--\ref{fig:training-curves} summarize the same behavior from complementary views. The parameter-count and energy-rank plots show that fully trainable models are comparable after optimization, while the frozen-hidden comparison consistently favors the proposed SVD-based initialization. Fig.~\ref{fig:diagnostics} provides a direct diagnostic view through confusion matrices and error galleries. The remaining errors are mostly plausible class confusions, and the before-training confusion matrices show that the proposed SVD-based initialization already induces structured predictions before gradient updates on MNIST and Fashion-MNIST. On the more challenging CIFAR-10 benchmark, this structure is weaker, but the initialized model remains above chance and avoids the degenerate single-class prediction pattern observed under Xavier initialization. After full training, both initializations reach comparable accuracy as in the other cases.

\begin{figure}[!p]
\centering
\safeincludegraphics[width=0.8\linewidth]{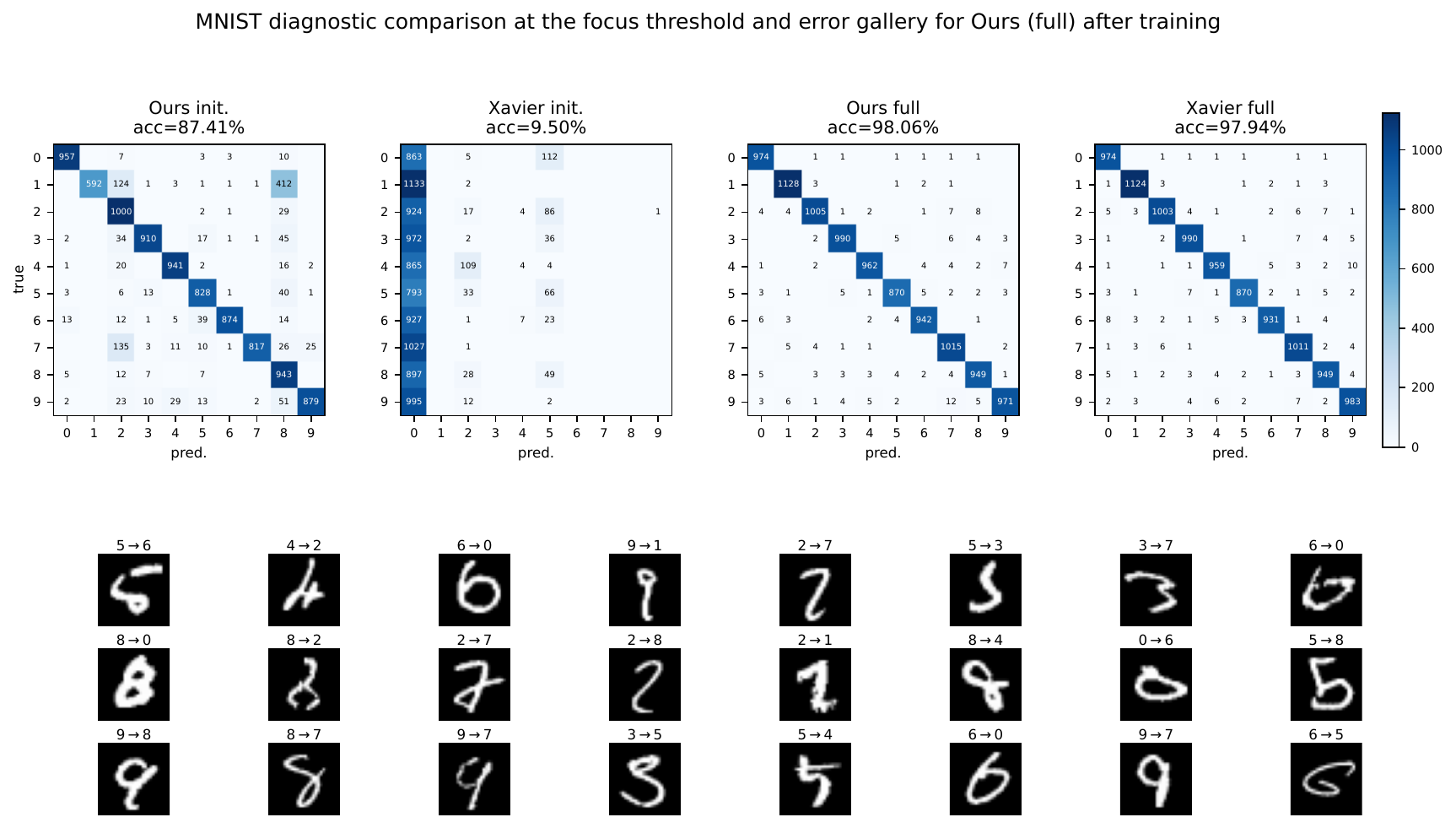}
\safeincludegraphics[width=0.8\linewidth]{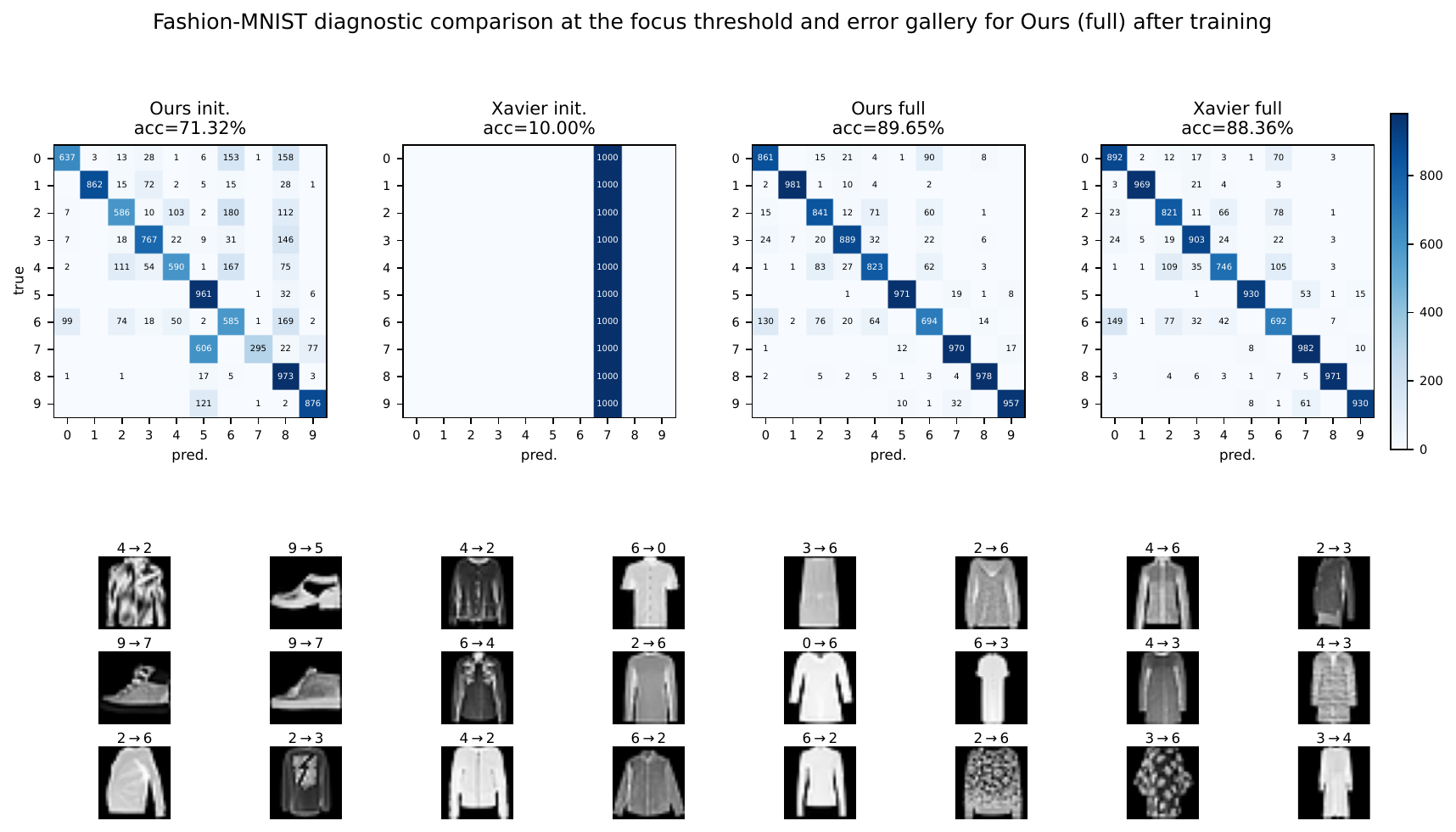}
\safeincludegraphics[width=0.8\linewidth]{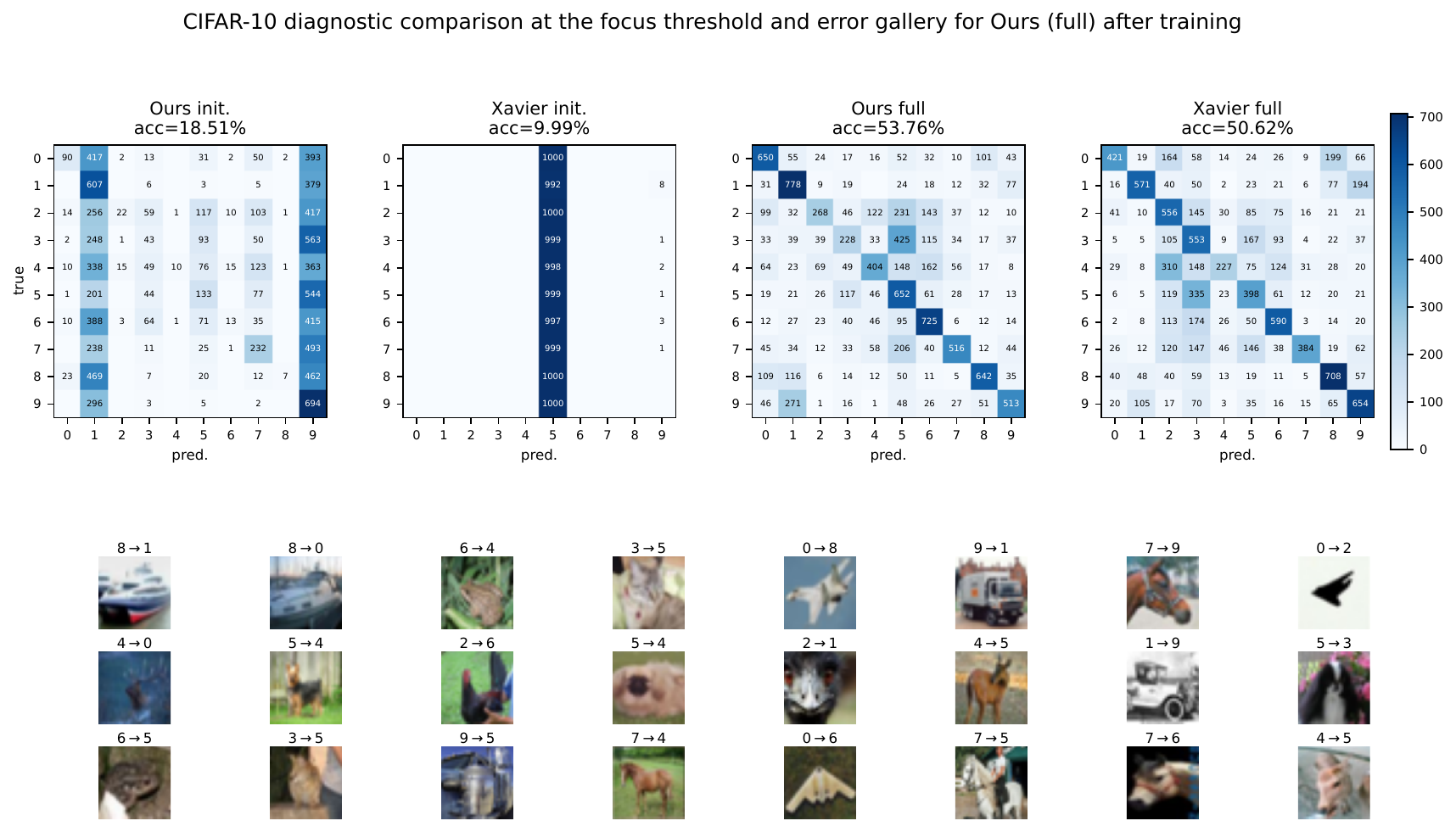}
\caption{Diagnostic comparison on MNIST, Fashion-MNIST, and CIFAR-10, shown from top to bottom at the common threshold \(\tau=0.95\). The confusion matrices and error galleries are shown for seed 0. For each dataset, the four confusion matrices compare the proposed initialization and the matched task-agnostic Xavier reference before training and after full training. Across datasets, the fully trained models reach comparable accuracy, while the SVD-based initialization already induces structured predictions before optimization, especially on MNIST and Fashion-MNIST.}
\label{fig:diagnostics}
\end{figure}

\section{Discussion and Limitations}
\label{sec:discussion}

The proposed construction connects two levels of geometry. At the UAT level, the half-space gates and finite covers illustrate how a sigmoidal MLP can be programmed from a target region when the region is explicitly known. This part is motivational and it does not provide a formal guarantee for the high-dimensional image initializer. At the practical image level, class-wise SVD estimates dominant directions and spectral scales from data. The energy threshold \(\tau\) becomes a practical complexity knob: small \(\tau\) gives a coarse class model and fewer gates, while large \(\tau\) retains finer variation and uses more parameters.

The cross-dataset results support this interpretation. On MNIST and Fashion-MNIST, the initialization is already discriminative before training, and frozen-hidden results clearly improve over frozen Xavier gates. Neutral-head ablations further show that this advantage is carried mainly by the fixed SVD gates rather than the aligned output head alone. On CIFAR-10, the effect is weaker but still informative under flattened raw pixels.

The non-neural SVD-Mahalanobis subspace classifier also provides a useful diagnostic for the applicability of the proposed SVD-based initialization. It uses the same class-wise means, spectral directions, spectral scales, and retained ranks before they are compiled into sigmoid slab gates. Therefore, its performance indicates whether the estimated input-space spectral geometry is already discriminative. This helps explain the cross-dataset trend: the method is strongest on MNIST and Fashion-MNIST, where the spectral reference is already informative, and weaker on CIFAR-10, where class identity depends more on local texture, pose, background, and higher-level visual features. The CIFAR-10 result therefore delineates the current regime of the method.

There are also clear limitations. The method is supervised and data-dependent, so it is not a drop-in replacement for task-agnostic initializers such as Xavier when labels or representative training data are unavailable. The SVD geometry is class-wise and linear; it does not model nonlinear manifolds inside each class. The slab construction treats retained directions independently and therefore approximates an axis-aligned spectral polytope in each class basis rather than the full data distribution. Finally, all experiments use low-resolution benchmarks and one-hidden-layer sigmoidal MLPs. A natural next step is to transfer the initialization principle to learned feature spaces and architectures with stronger visual inductive bias, such as convolutional neural networks~\cite{LeCun1998,He2016} and Vision Transformers~\cite{Dosovitskiy2021ViT}.

\section{Conclusion}

We presented a geometry-aware initialization framework for one-hidden-layer sigmoidal MLPs. Starting from the finite-sum form of UAT, we recalled how smooth half-space and cover gates can be constructed when target geometry is explicitly known. For image data, where such geometry is not given, we estimate it from labeled samples through class-wise SVD, use a SVD-Mahalanobis classifier as a non-neural geometric reference, and compile the retained spectral directions into pairs of sigmoid slab gates. This yields an explicit data-geometry-to-weights procedure for initializing the hidden layer of a sigmoidal network.

Overall, the proposed method is a supervised, data-dependent initializer. Full training reaches comparable final accuracy to Xavier, while the zero-epoch and frozen-hidden experiments show that class-wise SVD gates provide useful structure before optimization, even when the aligned output head is replaced by a neutral Xavier head. The weaker CIFAR-10 results delineate the current scope of the method and motivate extensions to learned feature spaces and architectures with stronger visual inductive bias.


\bibliographystyle{splncs04}
\bibliography{main}
\end{document}